\documentclass[sigconf]{acmart}

\AtBeginDocument{%
  \providecommand\BibTeX{{%
    \normalfont B\kern-0.5em{\scshape i\kern-0.25em b}\kern-0.8em\TeX}}}

\usepackage{microtype}
\usepackage{graphicx}
\usepackage{subfigure}
\usepackage{booktabs} 
\usepackage{tikz}
\usetikzlibrary{decorations.pathreplacing,angles,quotes}

\usepackage{amsthm}
\usepackage{enumitem}
\usepackage{wrapfig}
\usepackage{multirow}
\usepackage{lscape}
\usepackage{url}
\usepackage{algorithm,algorithmic}
\usepackage{balance}

\usepackage{amsmath,amsfonts,bm}

\def\eqref#1{equation~\ref{#1}}

\def\ceil#1{\lceil #1 \rceil}

\def\1{\bm{1}}

\def\vb{{\bm{b}}}

\def\ve{{\bm{e}}}

\def\vw{{\bm{w}}}
\def\vx{{\bm{x}}}
\def\vy{{\bm{y}}}
\def\vz{{\bm{z}}}

\def\mA{{\bm{A}}}

\def\mQ{{\bm{Q}}}
\def\mR{{\bm{R}}}

\def\mW{{\bm{W}}}

\DeclareMathAlphabet{\mathsfit}{\encodingdefault}{\sfdefault}{m}{sl}
\SetMathAlphabet{\mathsfit}{bold}{\encodingdefault}{\sfdefault}{bx}{n}

\def\sN{{\mathbb{N}}}

\def\sR{{\mathbb{R}}}

\def\sZ{{\mathbb{Z}}}

\newcommand{\R}{\mathbb{R}}

\newtheorem{theorem}{Theorem}
\newtheorem{definition}{Definition}

\copyrightyear{2020}
\acmYear{2020}
\setcopyright{acmlicensed}
\acmConference[KDD '20]{Proceedings of the 26th ACM SIGKDD Conference on Knowledge Discovery and Data Mining}{August 23--27, 2020}{Virtual Event, CA, USA}
\acmBooktitle{Proceedings of the 26th ACM SIGKDD Conference on Knowledge Discovery and Data Mining (KDD '20), August 23--27, 2020, Virtual Event, CA, USA}
\acmDOI{10.1145/3394486.3403059}
\acmISBN{978-1-4503-7998-4/20/08}

\begin{document}
\fancyhead{}

\title{Compositional Embeddings Using Complementary Partitions for Memory-Efficient Recommendation Systems}

\author{Hao-Jun Michael Shi}
\affiliation{%
  \institution{Northwestern University}
  \city{Evanston}
  \state{Illinois}
}
\email{hjmshi@u.northwestern.edu}

\author{Dheevatsa Mudigere}
\affiliation{%
  \institution{Facebook}
  \city{Menlo Park}
  \state{California}  
}
\email{dheevatsa@fb.com}

\author{Maxim Naumov}
\affiliation{%
  \institution{Facebook}
  \city{Menlo Park}
  \state{California}  
}
\email{mnaumov@fb.com}

\author{Jiyan Yang}
\affiliation{%
  \institution{Facebook}
  \city{Menlo Park}
  \state{California}  
}
\email{chocjy@fb.com}

\renewcommand{\shortauthors}{Shi, et al.}

\begin{abstract}
Modern deep learning-based recommendation systems exploit hundreds to thousands of different categorical features, each with millions of different categories ranging from clicks to posts. To respect the natural diversity within the categorical data, embeddings map each category to a unique dense representation within an embedded space. Since each categorical feature could take on as many as tens of millions of different possible categories, the embedding tables form the primary memory bottleneck during both training and inference. We propose a novel approach for reducing the embedding size in an end-to-end fashion by exploiting \textit{complementary partitions} of the category set to produce a unique embedding vector for each category without explicit definition. By storing multiple smaller embedding tables based on each complementary partition and combining embeddings from each table, we define a unique embedding for each category at smaller memory cost. This approach may be interpreted as using a specific fixed codebook to ensure uniqueness of each category's representation.  Our experimental results demonstrate the effectiveness of our approach over the hashing trick for reducing the size of the embedding tables in terms of model loss and accuracy, while retaining a similar reduction in the number of parameters.
\end{abstract}




\maketitle

\section{Introduction}

The design of modern deep learning-based recommendation models (DLRMs) is challenging because of the need to handle a large number of categorical (or sparse) features. For personalization or click-through rate (CTR) prediction tasks, examples of categorical features could include users, posts, or pages, with hundreds or thousands of these different features \cite{naumov2019deep}. Within each categorical feature, the set of categories could take on many diverse meanings. For example, social media pages could contain topics ranging from sports to movies.

In order to exploit this categorical information, DLRMs utilize embeddings to map each category to a unique dense representation in an embedded space; see \cite{cheng2016wide,he2017neural,wang2017deep,guo2017deepfm,lian2018xdeepfm,zhou2018deepi,zhou2018deep,naumov2019dimensionality}. More specifically, given a set of categories $S$ and its cardinality $|S|$, each categorical instance is mapped to an indexed row vector in an embedding table $\mW \in \sR^{|S| \times D}$, as shown in Figure \ref{fig:embedding}. Rather than predetermining the embedding weights, it has been found that \textit{jointly} training the embeddings with the rest of the neural network is effective in producing accurate models. 

Each categorical feature, however, could take on as many as tens of millions of different possible categories (i.e., $|S| \approx 10^7$), with an embedding vector dimension $D \approx 100$. Because of the vast number of categories, the number of embedding vectors form the primary memory bottleneck within both DLRM training and inference since each table could require multiple GBs to store.\footnote{Note that the relative dimensionality significantly differs from traditional language models, which use embedding vectors of length 100 to 500, with dictionaries of a maximum of hundreds of thousands of words.} 

One natural approach for reducing memory requirements is to decrease the size of the embedding tables by defining a hash function (typically the remainder function) that maps each category to an embedding index, where the embedding size is strictly smaller than the number of categories\footnote{We consider the case where the hashing trick is used primarily for reducing the number of categories. In practice, one may hash the categories for indexing and randomization purposes, then apply a remainder function to reduce the number of categories. Our proposed technique applies to the latter case.} \cite{weinberger2009feature}. However, this approach may blindly map vastly different categories to the same embedding vector, resulting in loss of information and deterioration in model quality. Ideally, one ought to reduce the size of the embedding tables while still producing a \textit{unique} representation for each category in order to respect the natural diversity of the data.

\begin{figure}
    \centering
    \begin{tikzpicture}
        \fill[blue!40!white] (0,0) rectangle (3,4);
        \draw [thick] (0,0) rectangle (3,4) node[color=white, pos=.5] {\Huge $W$};
        \draw [dashed] (0, 0.5) rectangle (3, 0.5);
        \draw [dashed] (0, 3.5) rectangle (3, 3.5);
        \draw [dashed] (0, 3) rectangle (3, 3);
        \draw [dashed] (0, 2.5) rectangle (3, 2.5);
        \draw[thick, decoration={brace,mirror,raise=5pt,amplitude=10pt},decorate] (3,0) -- node[right=15pt] {\Large $|S|$} (3,4);
        \draw[thick, decoration={brace,raise=5pt,amplitude=10pt},decorate] (0,4) -- node[above=15pt] {\Large $D$} (3, 4);
        
        \node at (-0.25, 3.75) {$0$};
        \node at (-0.25, 3.25) {$1$};
        \node at (-0.25, 2.75) {$2$};
        \node at (-0.7, 0.25) {$|S| - 1$};
    \end{tikzpicture}
    \caption{An embedding table.}
    \label{fig:embedding}
\end{figure}
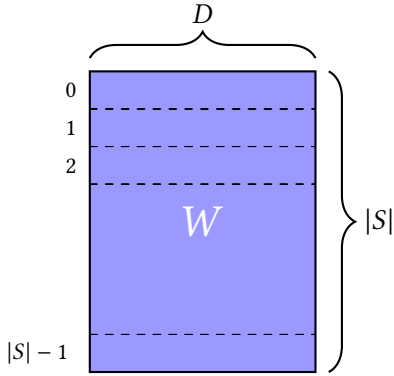


In this paper, we propose an approach for generating a unique embedding for each categorical feature by using \textit{complementary partitions} of the category set to generate \textit{compositional embeddings}, which interact multiple smaller embeddings to produce a final embedding. These complementary partitions could be obtained from inherent characteristics of the categorical data, or enforced artificially to reduce model complexity. We propose concrete methods for artificially defining these complementary partitions and demonstrate their usefulness on modified Deep and Cross (DCN) \cite{wang2017deep} and Facebook DLRM networks \cite{naumov2019deep} on the Kaggle Criteo Ad Display Challenge dataset. These methods are simple to implement, compress the model for both training and inference, do not require any additional pre- or post-training processing, and better preserve model quality than the hashing trick.

\subsection{Main Contributions}\
Our main contributions in this paper are as follows:

\begin{itemize}
    \item We propose a novel method called the \textit{quotient-remainder trick} for reducing the size of the embedding tables while still yielding a unique embedding vector for each category. The trick uses both the quotient and remainder functions to produce two different embeddings (whose embedding tables are of smaller size), and combines these embeddings to produce a final embedding, called a \textit{compositional embedding}. This reduces the number of embedding parameters from $O(|S| D)$ up to $O(\sqrt{|S|} D)$, where $|S|$ is the number of categories and $D$ is the embedding dimension.
    \item We generalize the quotient-remainder trick to compositional embeddings based on \textit{complementary partitions} of the category set. Complementary partitions require each category to be distinct from every other category according to at least one partition. This has the potential to reduce the number of embedding parameters to $O(k |S|^{1/k} D)$ where $k$ is the number of partitions.
    \item The experimental results demonstrate that our compositional embeddings yield better performance than the hashing trick, which is commonly used in practice. Although the best operation for defining the compositional embedding may vary, the element-wise multiplication operation produces embeddings that are most scalable and effective in general.
\end{itemize}

Section \ref{sec: simple example} will provide a simple example to motivate our framework for reducing model complexity by introducing the \textit{quotient-remainder trick}. In Section \ref{sec: comp partitions}, we will define complementary partitions, and provide some concrete examples that are useful in practice. Section \ref{sec: partition-decomposed embeddings} describes our proposed idea of compositional embeddings and clarifies the tradeoffs between this approach and using a full embedding table. Lastly, Section \ref{sec: experiments} gives our experimental results.

\section{Quotient-Remainder Trick}\label{sec: simple example}

Recall that in the typical DLRM setup, each category is mapped to a unique embedding vector in the embedding table. Mathematically, consider a single categorical feature and let $\varepsilon : S \rightarrow \{0, ..., |S| - 1\}$ denote an enumeration of $S$.\footnote{As an example, if the set of categories consist of $S = \{\text{dog}, \text{cat}, \text{mouse}\}$, then a potential enumeration of $S$ is $\varepsilon(\text{dog}) = 0$, $\varepsilon(\text{cat}) = 1$, and $\varepsilon(\text{mouse}) = 2$.} Let $\mW \in \R^{|S| \times D}$ be its corresponding embedding matrix or table, where $D$ is the dimension of the embeddings. We may encode each category (say, category $x \in S$ with index $i = \varepsilon(x)$) with a one-hot vector by $\ve_i \in \R^{|S|}$, then map this to a dense embedding vector $\vx_{\text{emb}} \in \R^D$ by
\begin{equation}
    \vx_{\text{emb}} = \mW^T \ve_i.
\end{equation}
Alternatively, the embedding may also be interpreted as a simple row lookup of the embedding table, i.e. $\vx_{\text{emb}} = \mW_{i,:}$. Note that this yields a memory complexity of $O(|S| D)$ for storing embeddings, which becomes restrictive when $|S|$ is large.

The naive approach of reducing the embedding table is to use a simple hash function \cite{weinberger2009feature}, such as the remainder function, called the \textit{hashing trick}. In particular, given an embedding table of size $m \in \sN$ where $m \ll |S|$, that is, $\widetilde{\mW} \in \R^{m \times D}$, one can define a hash matrix $\mR \in \R^{m \times |S|}$ by:
\begin{equation}
    \mR_{i,j} = 
    \begin{cases}
        1 & \text{if } j \bmod m = i \\
        0 & \text{otherwise.}
    \end{cases}
\end{equation}
Then the embedding is performed by:
\begin{equation}
    \vx_{\text{emb}} = \widetilde{\mW}^T \mR \ve_i.
\end{equation}
This process is summarized in Algorithm \ref{alg: hash trick}.

\begin{algorithm}
\caption{Hashing Trick}
\label{alg: hash trick}
\begin{algorithmic}
\REQUIRE Embedding table $\widetilde{\mW} \in \sR^{m \times D}$, category $x \in S$
\STATE Determine index $i = \varepsilon(x)$ of category $x$.
\STATE Compute hash index $j = i \bmod m$.
\STATE Look up embedding $\vx_{\text{emb}} = \widetilde{\mW}_{j, :}$.
\end{algorithmic}
\end{algorithm}

Although this approach significantly reduces the size of the embedding matrix from $O(|S| D)$ to $O(m D)$ since $m \ll |S|$, it naively maps multiple categories to the same embedding vector, resulting in loss of information and rapid deterioration in model quality. The key observation is that this approach does not yield a \textit{unique} embedding for each \textit{unique} category and hence does not respect the natural diversity of the categorical data in recommendation systems.

To overcome this, we propose the \textit{quotient-remainder trick}. Assume for simplicity that $m$ divides $|S|$ (although this does not have to hold in order for the trick to be applied). Let ``$\backslash$'' denote integer division or the quotient operation. Using two complementary functions -- the integer quotient and remainder functions -- we can produce two separate embedding tables and combine the embeddings in such a way that a unique embedding for each category is produced. This is formalized in Algorithm 2. 

\begin{algorithm}
\caption{Quotient-Remainder Trick}
\label{alg: quotient-remainder}
\begin{algorithmic}
\REQUIRE Embedding tables $\mW_1 \in \sR^{m \times D}$ and $\mW_2 \in \sR^{(|S| / m) \times D}$, category $x \in S$
\STATE Determine index $i = \varepsilon(x)$ of category $x$.
\STATE Compute hash indices $j = i \bmod m$ and $k = i \backslash m$.
\STATE Look up embeddings $\vx_{\text{rem}} = (\mW_1)_{j, :}$ and $\vx_{\text{quo}} = (\mW_2)_{k, :}$ .
\STATE Compute $\vx_{\text{emb}} = \vx_{\text{rem}} \odot \vx_{\text{quo}}$.
\end{algorithmic}
\end{algorithm}

More rigorously, define two embedding matrices: $\mW_1 \in \R^{m \times D}$ and $\mW_2 \in \R^{(|S| / m) \times D }$. Then define an additional hash matrix $\mQ \in \R^{(|S| / m) \times |S|}$
\begin{equation}
    \mQ_{i,j} = 
    \begin{cases}
        1 & \text{if } j \backslash m = i \\
        0 & \text{otherwise.}
    \end{cases}
\end{equation}
Then we obtain our embedding by
\begin{equation}
    \vx_{\text{emb}} = \mW_1^T \mR \ve_i \odot \mW_2^T \mQ \ve_i
\end{equation}
where $\odot$ denotes element-wise multiplication. This trick results in a memory complexity of $O(\frac{|S|}{m} D + m D)$, a slight increase in memory compared to the hashing trick but with the benefit of producing a unique representation. We demonstrate the usefulness of this method in our experiments in Section \ref{sec: experiments}. 


\section{Complementary Partitions}\label{sec: comp partitions}

The quotient-remainder trick is, however, only a single example of a more general framework for decomposing embeddings. Note that in the quotient-remainder trick, each operation (the quotient or remainder) partitions the set of categories into multiple ``buckets'' such that every index in the same ``bucket'' is mapped to the same vector. However, by combining embeddings from both the quotient and remainder together, one is able to generate a distinct vector for each index. 

Similarly, we want to ensure that each element in the category set may produce its own unique representation, even across multiple partitions. Using basic set theory, we formalize this concept to a notion that we call \textit{complementary partitions}. Let $[x]_P$ denote the equivalence class of $x \in S$ induced by partition $P$.\footnote{We slightly abuse notation by denoting the equivalence class by its partition rather than its equivalence relation for simplicity. For more details on set partitions, equivalence classes, and equivalence relations, please refer to the Appendix.}

\begin{definition}
Given set partitions $P_1, P_2, ..., P_k$ of set $S$, the set partitions are complementary if for all $a, b \in S$ such that $a \neq b$, there exists an $i$ such that $[a]_{P_i} \neq [b]_{P_i}$.
\end{definition}

As a concrete example, consider the set $S = \{0, 1, 2, 3, 4\}$. Then the following three set partitions are complementary: 
\begin{equation*}
\{ \{0\}, \{1, 3, 4\}, \{2\} \}, \{ \{0, 1, 3\}, \{2, 4\} \}, \{\{0, 3\}, \{1, 2, 4\} \}.
\end{equation*}
In particular, one can check that each element is distinct from every other element according to at least one of these partitions.  

Note that each equivalence class of a given partition designates a ``bucket'' that is mapped to an embedding vector. Hence, each partition corresponds to a single embedding table. Under complementary partitions, after each embedding arising from each partition is combined through some operation, each index is mapped to a distinct embedding vector, as we will see in Section \ref{sec: partition-decomposed embeddings}.

\subsection{Examples of Complementary Partitions}

Using this definition of complementary partitions, we can abstract the quotient-remainder trick and consider other more general complementary partitions. These examples are proved in the Appendix. For notational simplicity, we denote the set $\mathcal{E}(n) = \{0, 1, ..., n - 1\}$ for a given $n \in \sN$.

\begin{enumerate}
    \item Naive Complementary Partition: If 
    \begin{equation*}
        P = \{ \{x\} : x \in S \}
    \end{equation*}
    then $P$ is a complementary partition by definition. This corresponds to a full embedding table with dimension $|S| \times D$.
    \item Quotient-Remainder Complementary Partitions: Given $m \in \sN$, the partitions
    \begin{align*}
    P_1 & = \{ \{ x \in S : \varepsilon(x) \backslash m = l \} : l \in \mathcal{E}(\ceil{|S|/m}) \} \\
    P_2 & = \{ \{ x \in S : \varepsilon(x) \bmod m = l \} : l \in \mathcal{E}(m) \}
    \end{align*} 
    are complementary. This corresponds to the quotient-remainder trick in Section \ref{sec: simple example}.
    \item Generalized Quotient-Remainder Complementary Partitions: Given $m_i \in \sN$ for $i = 1, ..., k$ such that $|S| \leq \prod_{i = 1}^k m_i$ , we can recursively define complementary partitions 
    \begin{align*}
        P_1 & = \left\{ \{x \in S : \varepsilon(x) \bmod m_1 = l\} : l \in \mathcal{E}(m_1) \right\} \\ 
        P_j & = \left\{ \{ x \in S : \varepsilon(x) \backslash M_j \bmod m_j = l \} : l \in \mathcal{E}(m_j) \right\} 
    \end{align*}
    where $M_j = \prod_{i = 1}^{j - 1} m_i$ for $j = 2, ..., k$. This generalizes the quotient-remainder trick. 
    \item Chinese Remainder Partitions: Consider a pairwise coprime factorization greater than or equal to $|S|$, that is, $|S| \leq \prod_{i = 1}^k m_i$ for $m_i \in \sN$ for all $i = 1, ..., k$ and $\gcd(m_i, m_j) = 1$ for all $i \neq j$. Then we can define the complementary partitions
    \begin{equation*}
        P_j = \{ \{ x \in S : \varepsilon(x) \bmod m_j = l \} : l \in \mathcal{E}(m_j) \}
    \end{equation*}
    for $j = 1, ..., k$. 
\end{enumerate}

More arbitrary complementary partitions could also be defined depending on the application. Returning to our car example, one could define different partitions based on the year, make, type, etc. Assuming that the unique specification of these properties yields a unique car, these partitions would indeed be complementary. In the following section, we will demonstrate how to exploit this structure to reduce memory complexity.

\section{Compositional Embeddings Using Complementary Partitions}\label{sec: partition-decomposed embeddings}

Generalizing our approach in Section \ref{sec: simple example}, we would like to create an embedding table for each partition such that each equivalence class is mapped to an embedding vector. These embeddings could either be combined using some \textit{operation} to generate a compositional embedding or used directly as separate sparse features (which we call the \textit{feature generation} approach). The feature generation approach, although effective, may significantly increase the amount of parameters needed by adding additional features while not utilizing the inherent structure that the complementary partitions are formed from the same initial categorical feature.

More rigorously, consider a set of complementary partitions $P_1, P_2, ..., P_k$ of the category set $S$. For each partition $P_j$, we can create an embedding table $\mW_j \in \R^{|P_j| \times D_j}$ where each equivalence class $[x]_{P_j}$ is mapped to an embedding vector indexed by $i_j$ and $D_j \in \sN$ is the embedding dimension for embedding table $j$. Let $p_j : S \rightarrow \{0, ..., |P_j| - 1\}$ be the function that maps each element $x \in S$ to its corresponding equivalence class's embedding index, i.e. $x \mapsto i_j$. 

To generate our \textit{(operation-based) compositional embedding}, we interact all of the corresponding embeddings from each embedding table for our given category to obtain our final embedding vector
\begin{equation}
    \vx_{\text{emb}} = \omega(\mW_1^T \ve_{p_1(x)}, \mW_2^T \ve_{p_2(x)}, ..., \mW_k^T \ve_{p_k(x)} ) 
\end{equation}
where $\omega : \sR^{D_1} \times ... \times \sR^{D_k} \rightarrow \R^D$ is an operation function. Examples of the operation function include (but are not limited to): 
\begin{enumerate}
    \item Concatenation: Suppose $D = \sum_{i = 1}^k D_i$, then $\omega (\vz_1, ..., \vz_k) = [\vz_1^T, ..., \vz_k^T]^T$.
    \item Addition: Suppose $D_j = D$ for all $j$, then $\omega (\vz_1, ..., \vz_k) = \vz_1 + ... + \vz_k$.
    \item Element-wise Multiplication: Suppose $D_j = D$ for all $j$, then $\omega (\vz_1, ..., \vz_k) = \vz_1 \odot ... \odot \vz_k$\footnote{This is equivalent to factorizing the embeddings into the product of tensorized embeddings, i.e. if $\mW \in \sR^{|P_1| \times ... \times |P_k| \times D}$ is a $(k + 1)$-dimensional tensor containing all embeddings and $\mW_j \in \sR^{|P_j| \times D}$ for $j = 1, ..., k$ is the embedding table for partition $P_j$, then
    \begin{equation*}
        \mW[:, ..., :, d] = \mW_1[:, d] \otimes ... \otimes \mW_k[:, d]
    \end{equation*}
    for $d$ fixed, where $\otimes$ denotes the tensor outer product. This is similar to \cite{khrulkov2019tensorized} but instead applied vector-wise rather than component-wise.}.
\end{enumerate}

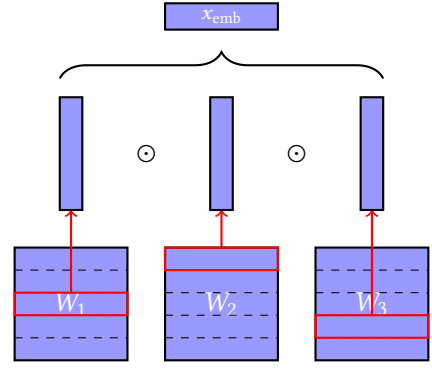
\begin{figure}
    \centering
    \begin{tikzpicture}
        \fill[blue!40!white] (0,0) rectangle (1.5,1.5);
        \draw[thick] (0,0) rectangle (1.5,1.5) node[color=white, pos=.5] {\Large $W_1$};
        \draw[dashed] (0, 1.2) -- (1.5, 1.2);
        \draw[dashed] (0, 0.9) -- (1.5, 0.9);
        \draw[dashed] (0, 0.6) -- (1.5, 0.6);
        \draw[dashed] (0, 0.3) -- (1.5, 0.3);
        
        \fill[blue!40!white] (2,0) rectangle (3.5,1.5);
        \draw[thick] (2,0) rectangle (3.5,1.5) node[color=white, pos=.5] {\Large $W_2$};
        \draw[dashed] (2, 1.2) -- (3.5, 1.2);
        \draw[dashed] (2, 0.9) -- (3.5, 0.9);
        \draw[dashed] (2, 0.6) -- (3.5, 0.6);
        \draw[dashed] (2, 0.3) -- (3.5, 0.3);

        \fill[blue!40!white] (4,0) rectangle (5.5,1.5);
        \draw[thick] (4,0) rectangle (5.5,1.5) node[color=white, pos=.5] {\Large $W_3$};
        \draw[dashed] (4, 1.2) -- (5.5, 1.2);
        \draw[dashed] (4, 0.9) -- (5.5, 0.9);
        \draw[dashed] (4, 0.6) -- (5.5, 0.6);
        \draw[dashed] (4, 0.3) -- (5.5, 0.3);
        
        \draw[thick, color=red] (0, 0.9) rectangle (1.5, 0.6);
        \draw[thick, color=red] (2, 1.5) rectangle (3.5, 1.2);
        \draw[thick, color=red] (4, 0.6) rectangle (5.5, 0.3);
        
        \draw[->, thick, color=red] (0.75, 0.9) -- (0.75, 2);
        \draw[->, thick, color=red] (2.75, 1.5) -- (2.75, 2);
        \draw[->, thick, color=red] (4.75, 0.6) -- (4.75, 2);

        \fill[blue!40!white] (0.6, 2) rectangle (0.9,3.5);
        \draw[thick] (0.6, 2) rectangle (0.9 , 3.5);
        \fill[blue!40!white] (2.6, 2) rectangle (2.9,3.5);
        \draw[thick] (2.6, 2) rectangle (2.9 , 3.5);
        \fill[blue!40!white] (4.6, 2) rectangle (4.9,3.5);
        \draw[thick] (4.6, 2) rectangle (4.9 , 3.5);
        
        \node at (1.75, 2.75) {\Large $\odot$} ;
        \node at (3.75, 2.75) {\Large $\odot$} ;

        \draw[thick, decoration={brace, amplitude=10pt},decorate] (0.6,3.75) -- (4.9, 3.75);
        
        \fill[blue!40!white] (2, 4.4) rectangle (3.5, 4.75);
        \draw[thick] (2, 4.4) rectangle (3.5, 4.75) node[color=white, pos=.5] {\small $x_{\text{emb}}$};

    \end{tikzpicture}
    \caption{Visualization of compositional embeddings with element-wise multiplication operation. The red arrows denote the selection of the embedding vector for each embedding table.}
    \label{fig:operation-based embeddings}
\end{figure}

One can show that this approach yields a unique embedding for each category under simple assumptions. We show this in the following theorem (proved in the Appendix). For simplicity, we will restrict ourselves to the concatenation operation.

\begin{theorem}\label{thm: uniqueness}
Assume that the vectors in each embedding table $\mW_j = \left[ \vw_1^j, ..., \vw_{|P_j|}^j \right]^T$ are distinct, that is $\vw_i^j \neq \vw_{\hat{i}}^j$ for $i \neq \hat{i}$ for all $j = 1, ..., k$. If the concatenation operation is used, then the compositional embedding of any category is unique, i.e. if $x, y \in S$ and $x \neq y$, then $\vx_{\text{emb}} \neq \vy_{\text{emb}}$.
\end{theorem}

This approach reduces the memory complexity of storing the entire embedding table $O(|S|D)$ to $O(|P_1|D_1 + |P_2|D_2 + ... + |P_k|D_k)$. Assuming $D_1 = D_2 = ... = D_k = D$ and $|P_j|$ can be chosen arbitrarily, this approach yields an optimal memory complexity of $O(k |S|^{1/k} D)$, a stark improvement over storing and utilizing the full embedding table. This approach is visualized in Figure \ref{fig:operation-based embeddings}.

\subsection{Path-Based Compositional Embeddings}

An alternative approach for generating embeddings is to define a different set of transformations for each partition (aside from the first embedding table). In particular, we can use a single partition to define an initial embedding table then pass our initial embedding through a composition of functions determined by the other partitions to obtain our final embedding vector. 

More formally, given a set of complementary partitions $P_1, P_2, ..., P_k$ of the category set $S$, we can define an embedding table $\mW \in \R^{|P_1| \times D_1}$ for the first partition, then define sets of functions $\mathcal{M}_j = \{ M_{j, i} : \R^{D_{j - 1}} \rightarrow \R^{D_j} : i \in \{1, ..., |P_j|\} \}$ for every other partition. As before, let $p_j : S \rightarrow \{1, ..., |P_j|\}$ be the function that maps each category to its corresponding equivalence class's embedding index. 

To obtain the embedding for category $x \in S$, we can perform the following transformation
\begin{equation}
    \vx_{\text{emb}} = (M_{k, p_k(x)} \circ ... \circ M_{2, p_2(x)})( \mW \ve_{p_1(x)} ).
\end{equation}
We call this formulation of embeddings \textit{path-based compositional embeddings} because each function in the composition is determined based on the unique set of equivalence classes from each partition, yielding a unique ``path'' of transformations. These transformations may contain parameters that also need to be trained concurrently with the rest of the network. Examples of the function $M_{j, i}$ could include:
\begin{enumerate}
    \item Linear Function: If $\mA \in \R^{D_j \times D_{j - 1}}$ and $\vb \in \R^{D_j}$ are parameters, then $M_{j, i}(\vz) = \mA \vz + \vb$.
    \item Multilayer Perceptron (MLP): Let $L$ be the number of layers. Let $d_0 = D_{j - 1}$ and $d_L = D_j$ and $d_j \in \sN$ for $j = 1, ..., k - 1$ denote the number of nodes at each layer. Then if $\mA_1 \in \R^{d_1 \times d_0}$, $\mA_2 \in \R^{d_2 \times d_1}$, ..., $\mA_L \in \R^{d_L \times d_{L - 1}}$, $\vb_1 \in \R^{d_1}$, $\vb_2 \in \R^{d_2}$, ..., $\vb_L \in \R^{d_L}$ are parameters, and $\sigma : \R \rightarrow \R$ is an activation function (say, ReLU or sigmoid function) that is applied componentwise, then 
    \begin{equation*}
        M_{j, i}(\vz) = \mA_L \sigma( ... \mA_2 \sigma( \mA_1 \vz + \vb_1 ) + \vb_2 ... ) + \vb_L.
    \end{equation*}
\end{enumerate}

Unlike operation-based compositional embeddings, path-based compositional embeddings require non-embedding parameters within the function to be learned, which may complicate training. The reduction in memory complexity also depends on how these functions are defined and how many additional parameters they add. For linear functions or MLPs with small fixed size, one can maintain the $O(k |S|^{1/k} D)$ complexity. This is visualized in Figure \ref{fig:path-based embeddings}.

\begin{figure}
    \centering
    \begin{tikzpicture}

        \fill[blue!40!white] (2,0) rectangle (3.5,1.5);
        \draw[thick] (2,0) rectangle (3.5,1.5) node[color=white, pos=.5] {\Large $W$};
        \draw[dashed] (2, 1.2) -- (3.5, 1.2);
        \draw[dashed] (2, 0.9) -- (3.5, 0.9);
        \draw[dashed] (2, 0.6) -- (3.5, 0.6);
        \draw[dashed] (2, 0.3) -- (3.5, 0.3);

        \draw[thick, color=red] (2, 1.5) rectangle (3.5, 1.2);

        \fill[blue!40!white] (0,2.25) -- (0.75, 3.5) -- (1.5,2.25) -- (0, 2.25);
        \draw[thick] (0,2.25) -- (0.75, 3.5) -- (1.5,2.25) -- (0, 2.25) node[color=white, pos=.5, above=.5] {$\text{MLP}_1$};

        \fill[blue!40!white] (2,2.25) -- (2.75, 3.5) -- (3.5,2.25) -- (2, 2.25);
        \draw[thick] (2,2.25) -- (2.75, 3.5) -- (3.5,2.25) -- (2, 2.25) node[color=white, pos=.5, above=.5] {$\text{MLP}_2$};
        
        \fill[blue!40!white] (4,2.25) -- (4.75, 3.5) -- (5.5,2.25) -- (4, 2.25);
        \draw[thick] (4,2.25) -- (4.75, 3.5) -- (5.5,2.25) -- (4, 2.25) node[color=white, pos=.5, above=.5] {$\text{MLP}_3$};

        \fill[blue!40!white] (2, 4.25) rectangle (3.5, 4.55);
        \draw[thick] (2, 4.25) rectangle (3.5, 4.55) node[color=white, pos=.5] {\small $x_{\text{emb}}$};

        \draw[thick, red=color, ->] (2.75, 1.5) -- (4.75, 2.25);
        \draw[thick, red=color, ->] (4.75, 3.5) -- (2.75, 4.25);

    \end{tikzpicture}
    \caption{Visualization of path-based compositional embeddings. The red arrows denote the selection of the embedding vector and its corresponding path of transformations.}
    \label{fig:path-based embeddings}
\end{figure}
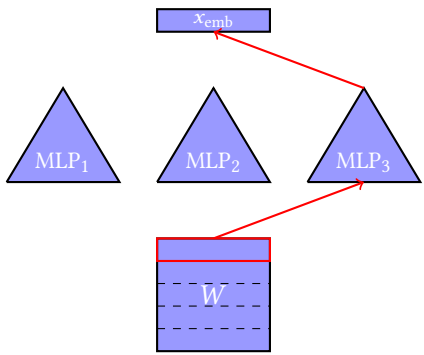

\section{Experiments}\label{sec: experiments}

In this section, we present a comprehensive set of experiments to test the quotient-remainder trick for reducing the number of parameters while preserving model loss and accuracy over many different operations. In particular, we show that quotient-remainder trick allows us to trade off model accuracy attained by full embedding tables with model size obtained with the hashing trick. 

For comparison, we consider both DCN \cite{wang2017deep} and Facebook DLRM networks. These two networks were selected as they are representative of most models for CTR prediction. We provide the model and experimental setup below.

\subsection{Model Specifications}


The DCN architecture considered in this paper consists of a deep network with 3 hidden layers consisting of 512, 256, and 64 nodes, respectively. The cross network consists of 6 layers. An embedding dimension of 16 is used across all categorical features.

The Facebook DLRM architecture consists of a bottom (or dense) MLP with 3 hidden layers with 512, 256, and 64 nodes, respectively, and an top (or output) MLP with 2 hidden layers consisting of 512 and 256 nodes. An embedding dimension of 16 is used. When thresholding, the concatenation operation uses an embedding dimension of 32 for non-compositional embeddings.

Note that in the baseline (using full embedding tables), the total number of rows in each table is determined by the cardinality of the category set for the table's corresponding feature.

\subsection{Experimental Setup and Data Pre-processing}

The experiments are performed on the Criteo Ad Kaggle Competition dataset\footnote{\scriptsize{\url{http://labs.criteo.com/2014/02/kaggle-display-advertising-challenge-dataset/}}} (Kaggle). Kaggle has 13 dense features and 26 categorical features. It consists of approximately 45 million datapoints sampled over 7 days. We use the first 6 days as the training set and split the 7th day equally into a validation and test set. The dense features are transformed using a log-transform. Unlabeled categorical features or labels are mapped to \texttt{NULL} or $0$, respectively.

Note that for this dataset, each category is preprocessed to map to its own index. However, it is common in practice to apply the hashing trick to map each category to an index in an online fashion and to randomize the categories prior to reducing the number of embedding rows using the remainder function. Our techniques may still be applied in addition to the initial hashing as a replacement to the remainder function to systematically reduce embedding sizes.

Each model is optimized using the Adagrad \cite{duchi2011adaptive} and AMSGrad \cite{kingma2014adam,reddi2019convergence} optimizers with their default hyperparameters; we choose the optimizer that yields the best validation loss. Single epoch training is used with a batch size of 128 and no regularization. All experiments are averaged over 5 trials. Both the mean and a single standard deviation are plotted. Here, we use an embedding dimension of 16. The model loss is evaluated using binary cross-entropy.

\begin{figure*}
    \centering
    \includegraphics[width=0.4\textwidth]{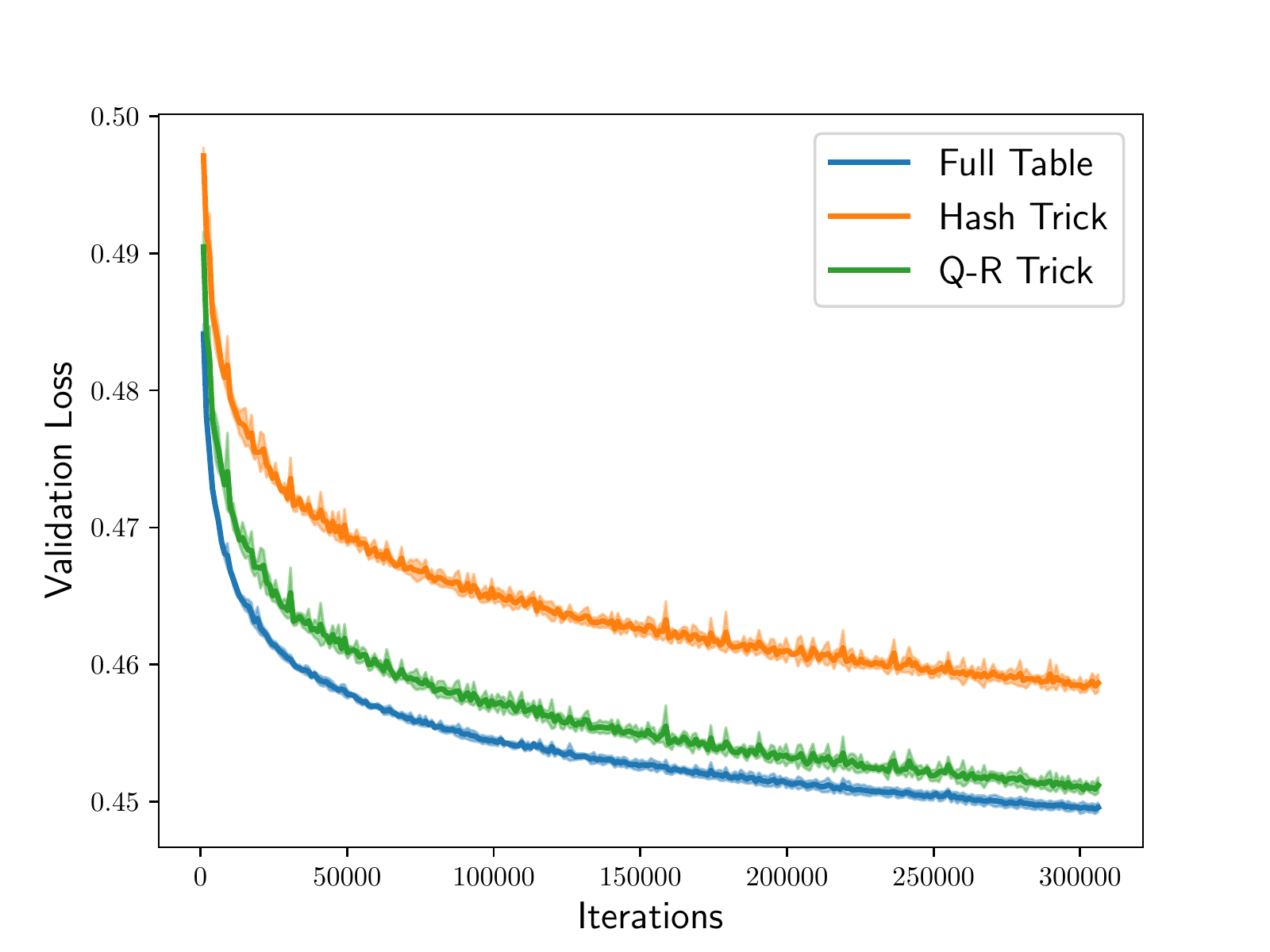}
    \includegraphics[width=0.4\textwidth]{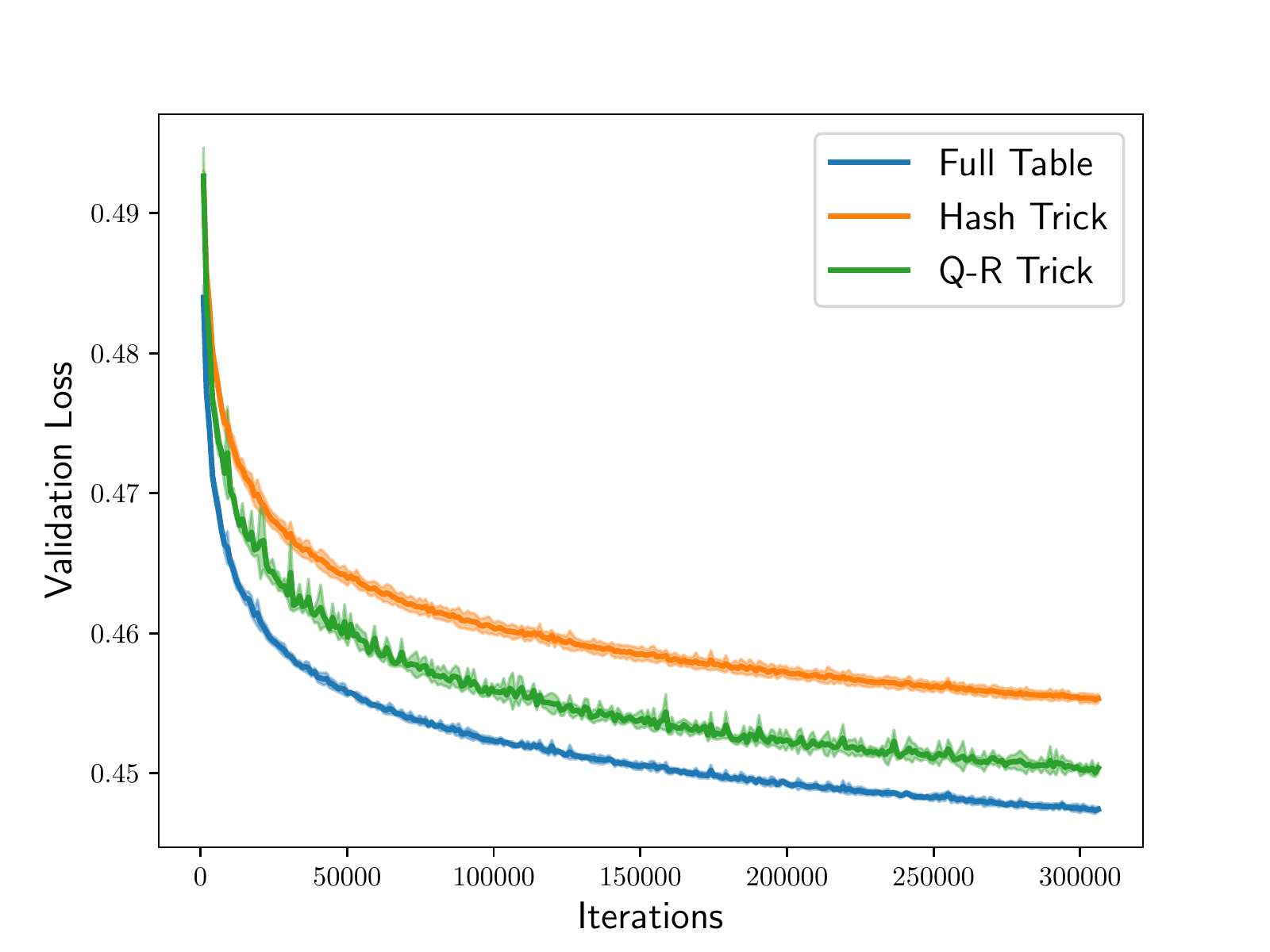}
    \caption{Validation loss against the number of iterations when training the DCN (left) and Facebook DLRM (right) networks over 5 trials. Both the mean and standard deviation are plotted. ``Full Table'' corresponds to the baseline using full embedding tables (without hashing), ``Hash Trick'' refers to the hashing trick, and ``Q-R Trick'' refers to the quotient-remainder trick (with element-wise multiplication). Note that the hashing trick and quotient-remainder trick result in an approximate $4 \times$ reduction in model size.}
    \label{fig:simple comparison}
\end{figure*}

To illustrate the quotient-remainder trick, we provide a simple comparison of the validation loss throughout training for full embedding tables, the hashing trick, and the quotient-remainder trick (with the element-wise multiplication) in Figure \ref{fig:simple comparison}. The model size reduction can be quite significant, since it scales directly with the the number of hash collisions. We enforce 4 hash collisions, yielding about a $4 \times$ reduction in model size. Each curve shows the average and standard deviation of the validation loss over 5 trials.

As expected, we see in Figure \ref{fig:simple comparison} that the quotient-remainder trick interpolates between the compression of the hashing trick and the accuracy attained by the full embedding tables. 

\subsection{Compositional Embeddings}

To provide a more comprehensive comparison, we vary the number of hash collisions enforced within each feature and plot the number of parameters against the test loss for each operation. We enforce between 2-7 and 60 hash collisions on each categorical feature. We plot our results in Figure \ref{fig:small hash collision}, where each point corresponds to the averaged result for a fixed number of hash collisions over 5 trials. In particular, since the number of embedding parameters dominate the total number of parameters in the entire network, the number of hash collisions is approximately inversely proportional to the number of parameters in the network.

\begin{figure*}[ht]
    \centering
    \includegraphics[width=0.4\textwidth]{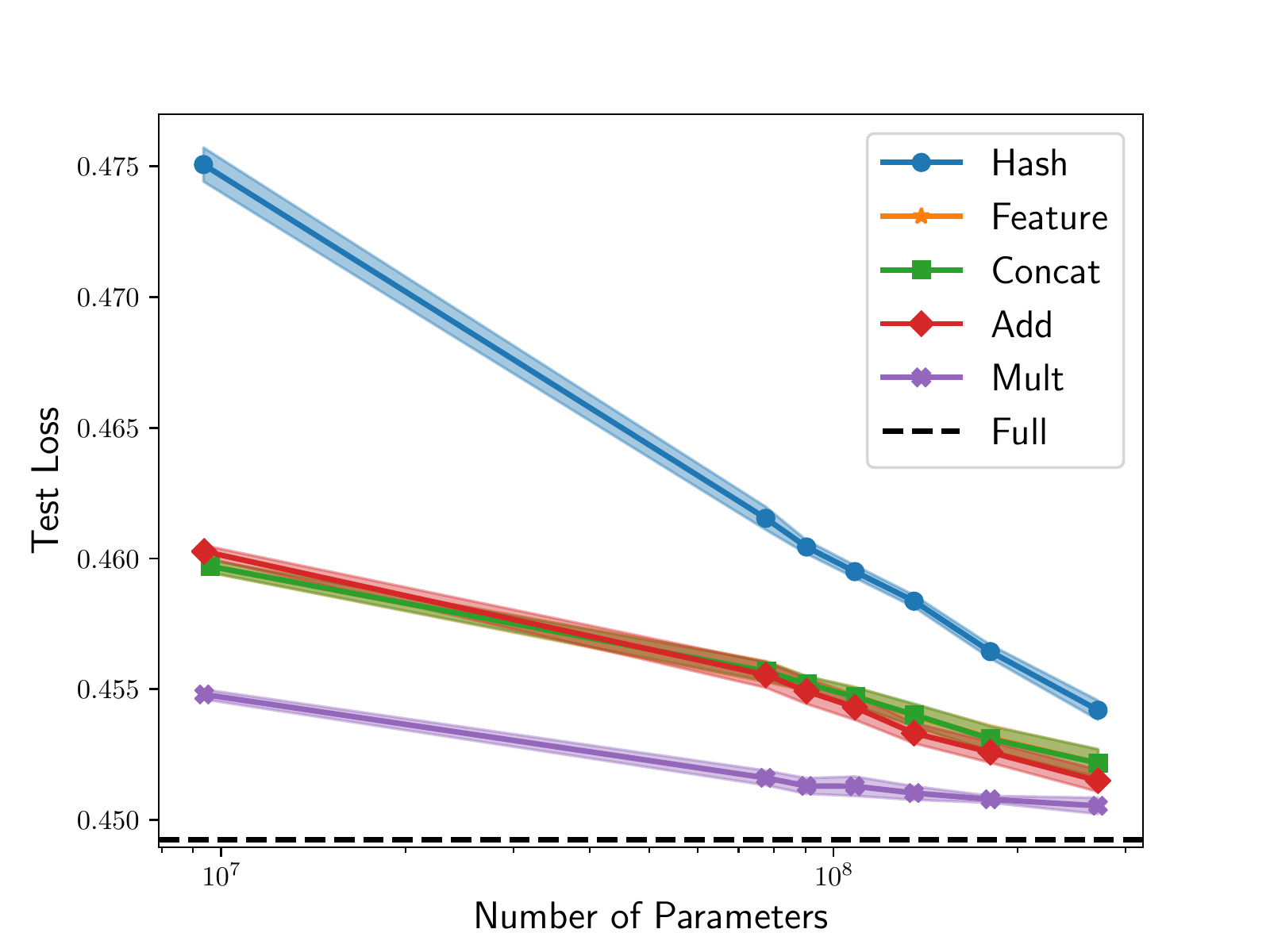}
    \includegraphics[width=0.4\textwidth]{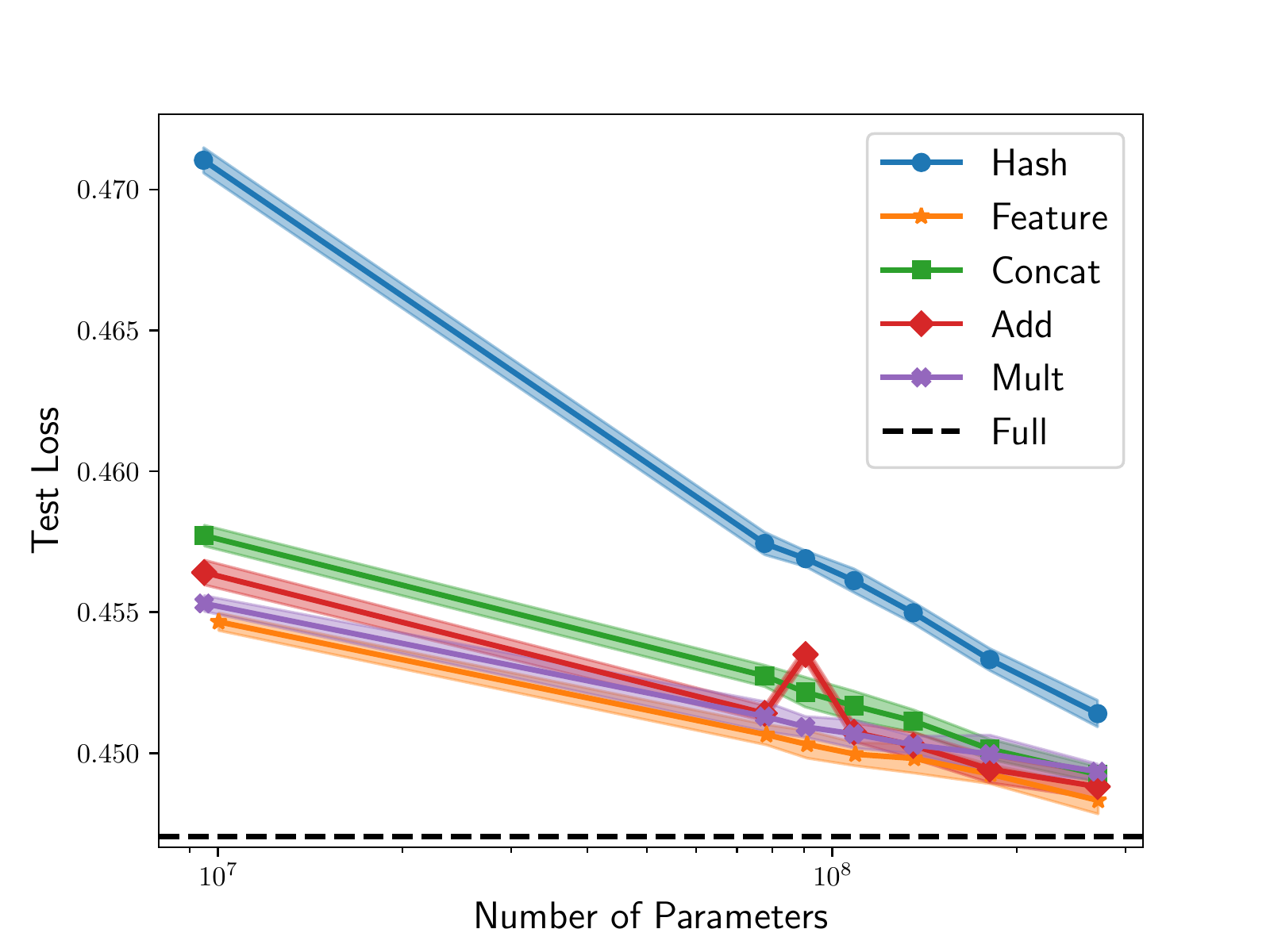}
    \caption{Test loss against the number of parameters for 2-7 and 60 hash collisions on DCN (left) and Facebook DLRM (right) networks over 5 trials. Both the mean and standard deviation are plotted. Hash, Feature, Concat, Add, and Mult correspond to different operations. Full corresponds to the baseline using full embedding tables (without hashing). The baseline model using full embedding tables contains approximately $5.4 \times 10^8$ parameters.}
    \label{fig:small hash collision}
\end{figure*}

The multiplication operation performs best overall, performing closely to the feature generation baseline which comes at the cost of an additional half-million parameters for the Facebook DLRM and significantly outperforming all other operations for DCN. Interestingly, we found that AMSGrad significantly outperformed Adagrad when using the multiplication operation. Compared to the hashing trick with 4 hash collisions, we were able to attain similar or better solution quality with up to 60 hash collisions, an approximately $15 \times$ smaller model. With up to 4 hash collisions, we are within 0.3\% of the baseline model for DCN and within 0.7\% of the baseline model for DLRM. Note that the baseline performance for DLRM outperforms DCN in this instance.

Because the number of categories within each categorical feature may vary widely, it may be useful to only apply the hashing trick to embedding tables with sizes larger than some threshold. To see the tradeoff due to thresholding, we consider the thresholds $\{1, 20, 200, 2000, 20000\}$ and plot the threshold number against the test loss for 4 hash collisions. For comparison, we include the result with the full embedding table as a baseline in Figure \ref{fig:small thresh}. We also examine the effect of thresholding on the total number of parameters in Figure \ref{fig: thresh parameters}.

\begin{figure*}[ht]
    \centering
    \includegraphics[width=0.4\textwidth]{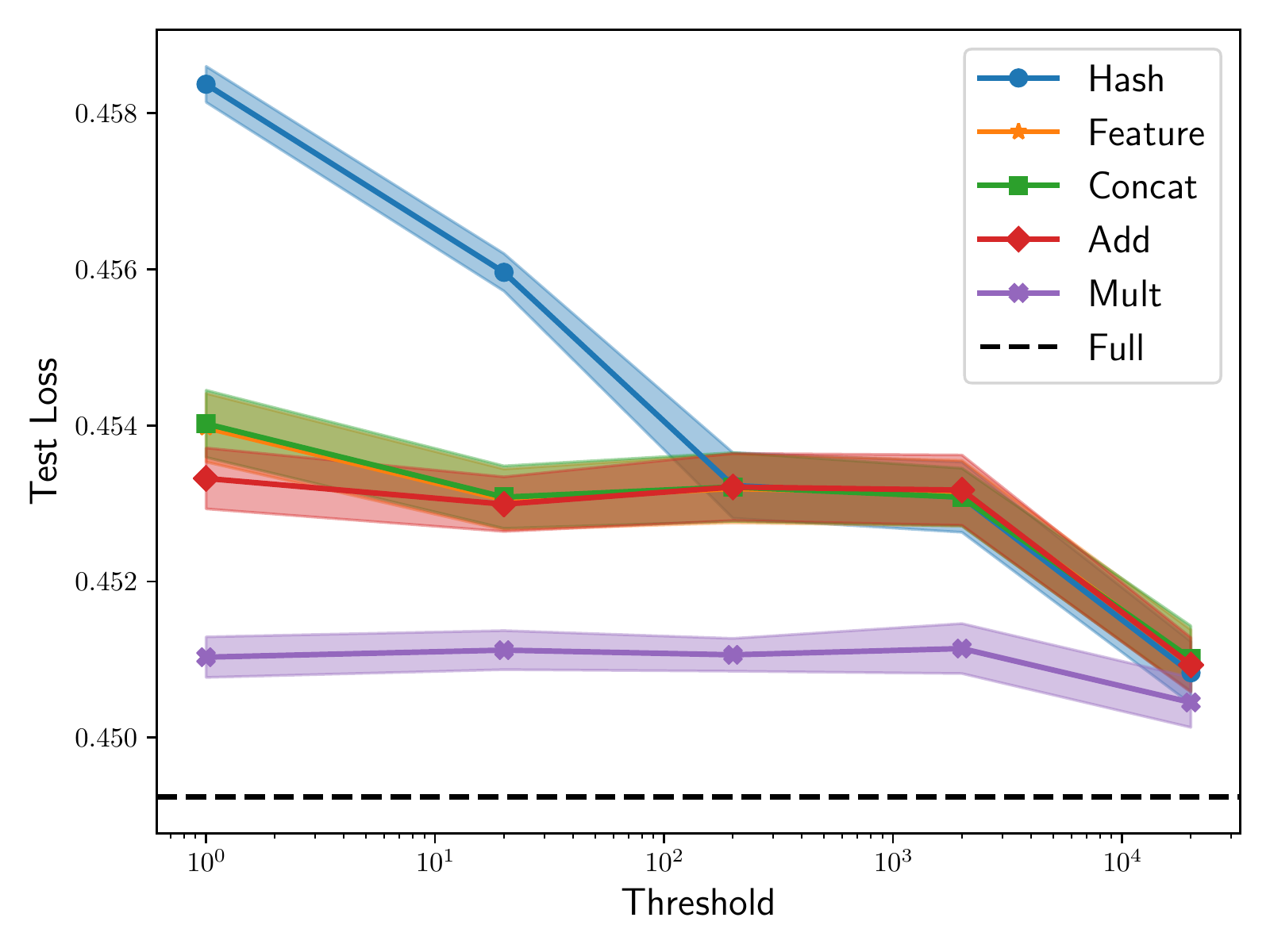}
    \includegraphics[width=0.4\textwidth]{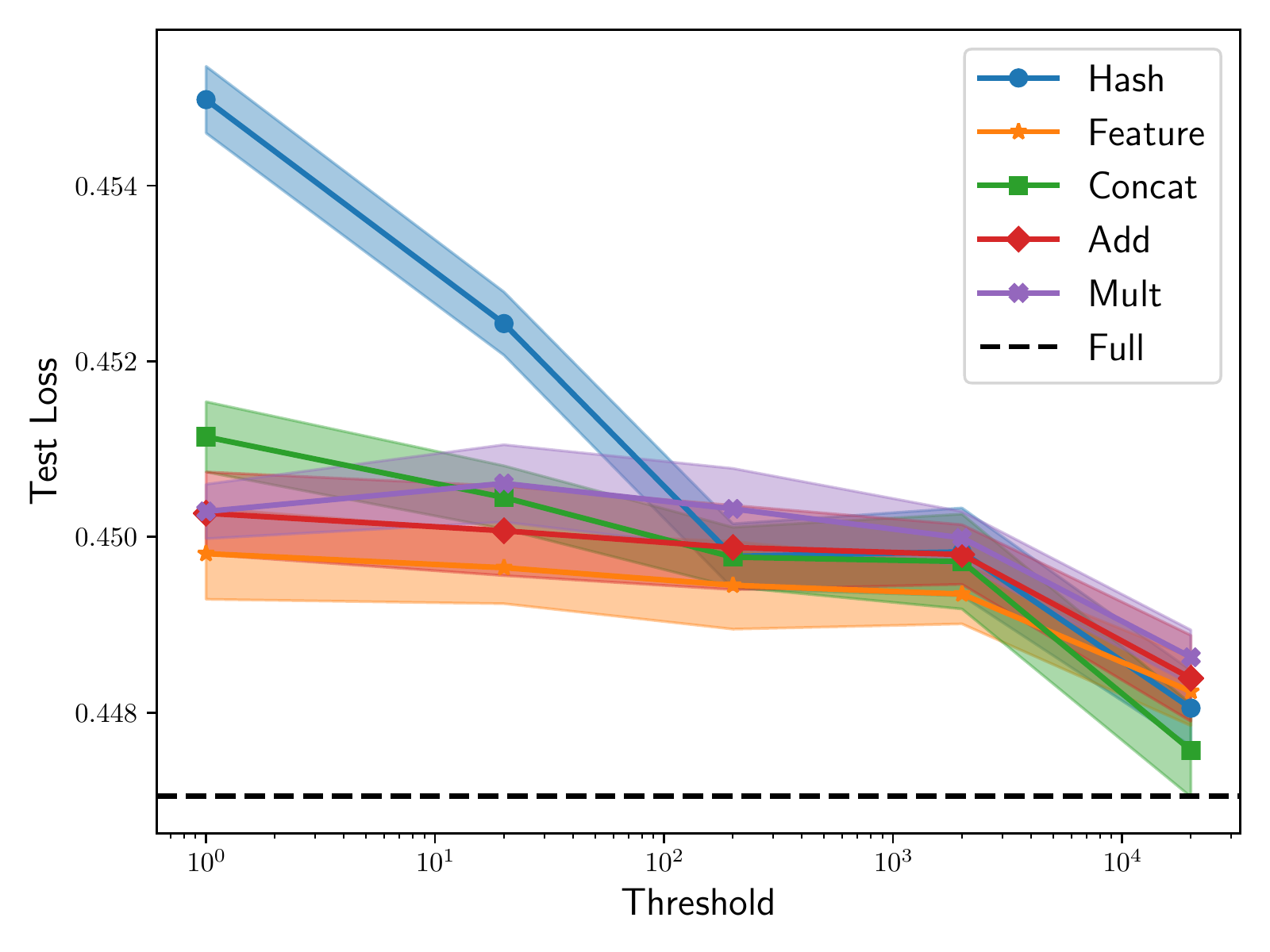}
    \caption{Test loss against the threshold number with 4 hash collisions on DCN (left) and Facebook DLRM (right) networks over 5 trials. Both the mean and standard deviation are plotted. Hash, Feature, Concat, Add, and Mult correspond to different operations. Full corresponds to the baseline using full embedding tables (without hashing). The baseline model using full embedding tables contains approximately $5.4 \times 10^8$ parameters.}
    \label{fig:small thresh}
\end{figure*}

\begin{figure*}
    \centering
    \includegraphics[width=0.4\textwidth]{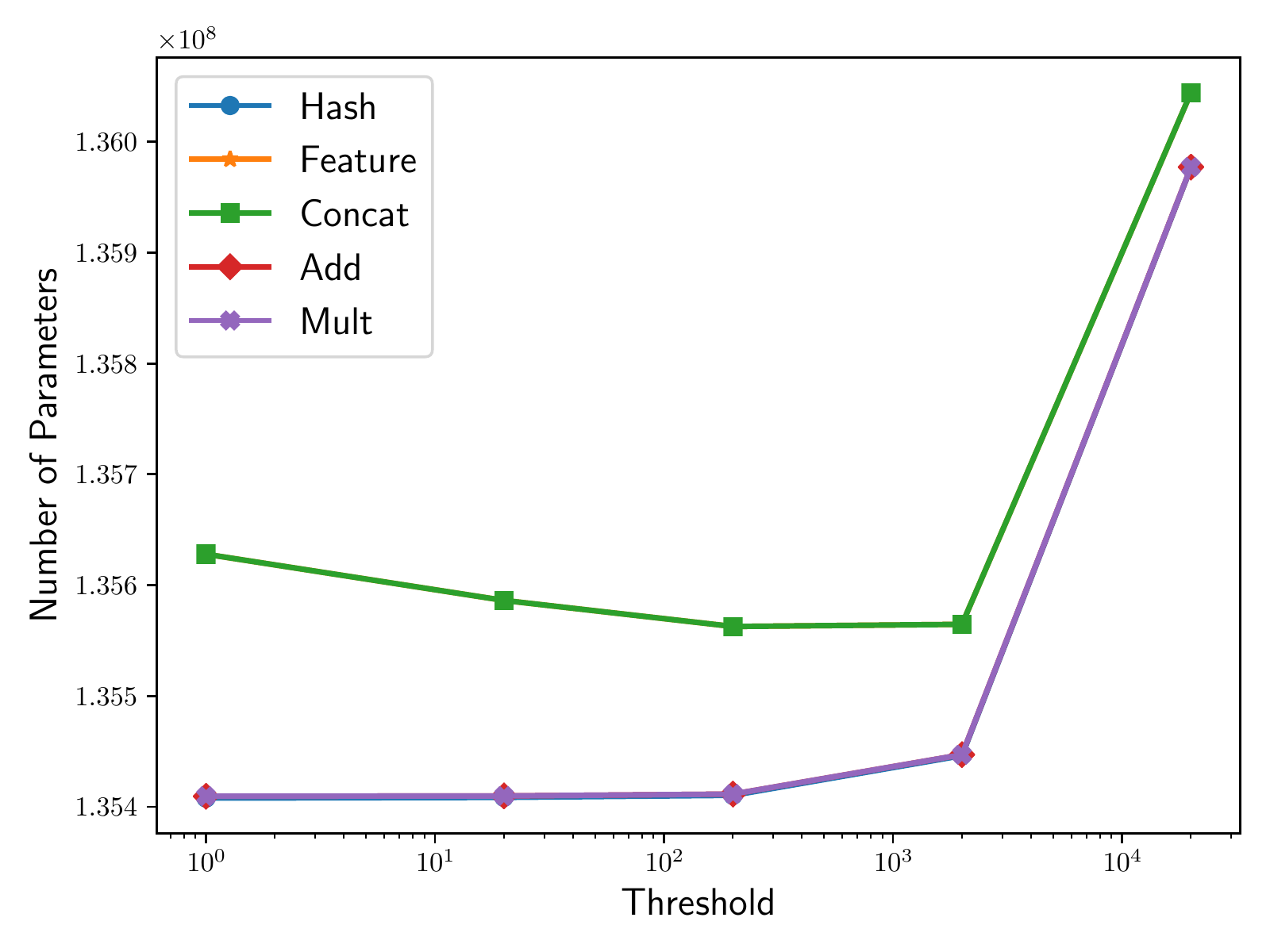}
    \includegraphics[width=0.4\textwidth]{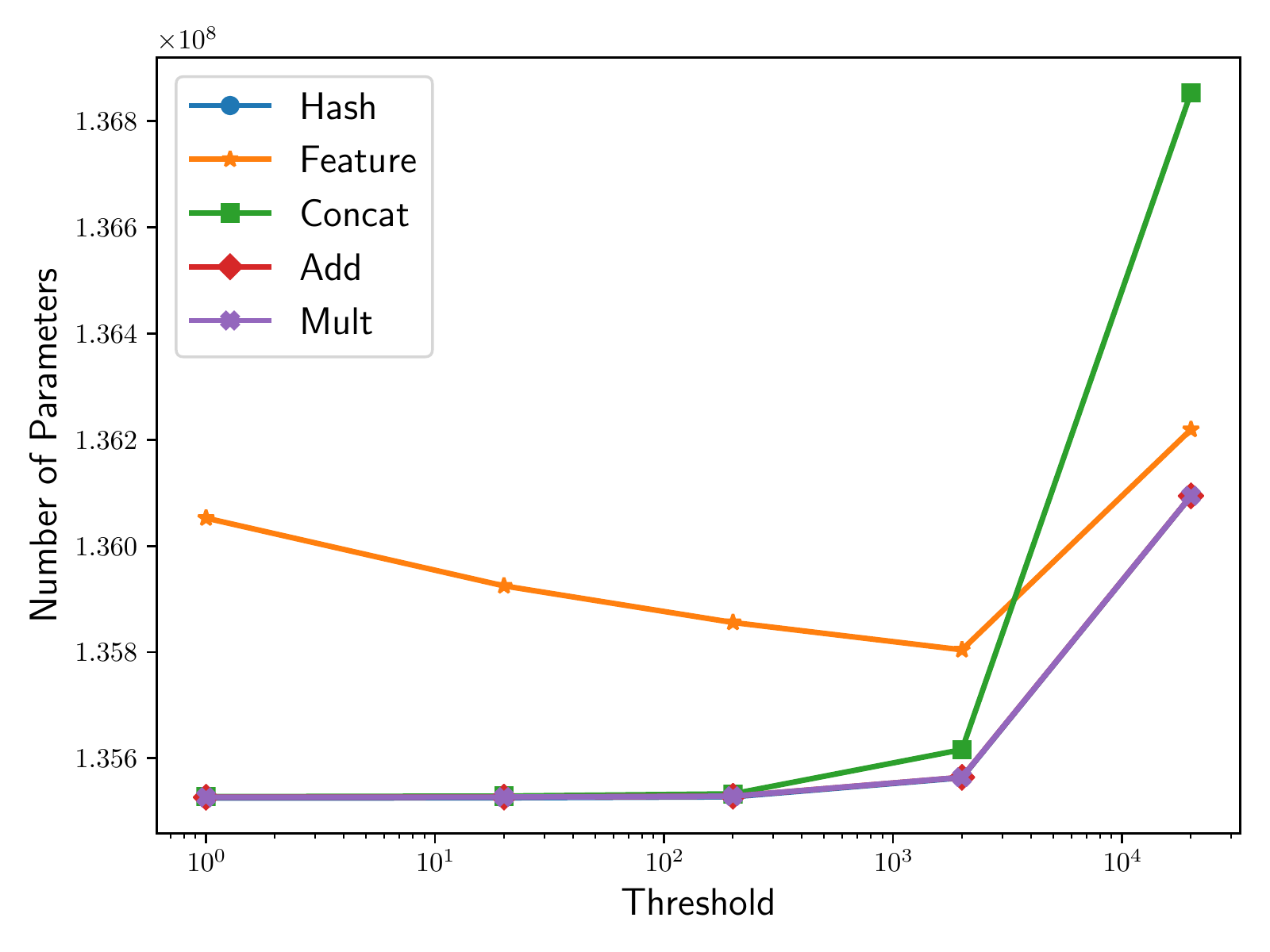}
    \caption{Number of parameters against the threshold number with 4 hash collisions on DCN (left) and Facebook DLRM (right) networks. Hash, Feature, Concat, Add, and Mult correspond to different operations. Full corresponds to the baseline using full embedding tables (without hashing). The baseline model using full embedding tables contains approximately $5.4 \times 10^8$ parameters.}
    \label{fig: thresh parameters}
\end{figure*}

We see that when thresholding is used, the results are much more nuanced and improvement in performance depends on the operation considered. In particular, we find that the element-wise multiplication works best for DCN, while the concatenation operation works better for Facebook DLRM. For DLRM, we were able to observe an improvement from a 0.7\% error to 0.5\% error to the baseline while maintaining an approximate $4 \times$ reduction in model size. 

\subsection{Path-Based Compositional Embeddings}

In the following set of preliminary experiments, we consider the quotient-remainder trick for path-based compositional embeddings. Here, we fix to 4 hash collisions and define an MLP with a single hidden layer of sizes 16, 32, 64, and 128. The results are shown in Table \ref{table: path-based compositional embeddings}.

\begin{table*}[ht]
    \centering
    \caption{Average test loss and number of parameters for different MLP sizes with 4 hash collisions over 5 trials.}
    \label{table: path-based compositional embeddings}
    \vskip 0.15in
    \begin{small}
    \begin{sc}
    \begin{tabular}{ll|llll}
    \toprule
    & \textbf{Hidden Layer} & 16        & 32        & 64        & 128       \\ \midrule
    \multirow{2}{*}{\textbf{DCN}} & \textbf{\# Parameters} & 135,464,410 & 135,519,322 & 135,629,146 & 135,848,794 \\
    & \textbf{Test Loss}            &  0.45263  & 0.45254  &  \textbf{0.45252}  &  0.4534   \\ \midrule
    \multirow{2}{*}{\textbf{DLRM}} & \textbf{\# Parameters} & 135,581,537 & 135,636,449 & 135,746,273 & 135,965,921 \\
    & \textbf{Test Loss}  &        0.45349  &  0.45312  &  \textbf{0.45306}  &  0.45651     \\
    \bottomrule
    \end{tabular}
    \end{sc}
    \end{small}
    \vskip -0.1in
\end{table*}

From Table \ref{table: path-based compositional embeddings}, we obtain an optimal hidden layer of size 64. The trend follows an intuitive tradeoff: using a smaller network may be easier to train but may not sufficiently transform the embeddings, while a larger network may have greater capacity to fit a more complex transformation but require more parameters to be learned. In this case, the modified DCN outperforms the modified DLRM network, although this is not true in general.

\subsection{Discussion of Tradeoffs}

As seen in Figure \ref{fig:small hash collision}, using the quotient-remainder trick enforces arbitrary structures on the categorical data which can yield arbitrary loss in performance, as expected. This yields a trade-off between memory and performance; a larger embedding table will yield better model quality, but at the cost of increased memory requirements. Similarly, using a more aggressive version of the quotient-remainder trick will yield smaller models, but lead to a reduction in model quality. Most models exponentially decrease in performance with the number of parameters. 

Both types of compositional embeddings reduce the number of parameters by implicitly enforcing some structure defined by the complementary partitions in the generation of each category's embedding. Hence, the quality of the model ought to depend on how closely the chosen partitions reflect intrinsic properties of the category set and their respective embeddings. In some problems, this structure may be identified; however, the dataset considered in this paper contains no additional knowledge of the categories. Since practitioners performing CTR prediction typically apply the hashing trick, our method clearly improves upon this baseline with only small additional cost in memory.

Path-based compositional embeddings also yield more compute-intensive models with the benefit of lower model complexity. Whereas the operation-based approach attempts to definitively operate on multiple coarse representations, path-based embeddings explicitly define the transformations on the representation, a more difficult but intriguing problem. Unfortunately, our preliminary experiments show that our current implementation of path-based compositional embeddings do not supersede operation-based compositional embeddings; however, we do believe that path-based compositional embeddings are potentially capable of producing improved results with improved modeling and training techniques, and are worthy of further investigation.

\section{Related Work}

Wide and deep models \cite{cheng2016wide} jointly train both a deep network and linear model to combine the benefits of memorization and generalization for recommendation systems. Factorization machines \cite{rendle2010factorization,rendle2012factorization} played a key role in the next step of development of DLRMs by identifying that sparse features (induced by nominal categorical data) could be appropriately exploited by interacting different dense representations of sparse features with an inner product to produce meaningful higher-order terms. Generalizing this observation, some recommendation models \cite{guo2017deepfm,he2017neural,wang2017deep,lian2018xdeepfm,zhou2018deep} jointly train a deep network with a specialized model in order to directly capture higher-order interactions in an efficient manner. The Facebook DLRM network \cite{naumov2019deep} mimics factorization machines more directly by passing the pairwise dot product between different embeddings into a multilayer perceptron (MLP). More sophisticated techniques that incorporate trees, memory, and (self-)attention mechanisms (to capture sequential user behavior) have also been proposed \cite{zheng2018mars,zhou2018atrank,zhou2018deepi,zhu2018learning}.

Towards the design of the embeddings, Naumov \cite{naumov2019dimensionality} proposed an approach for prescribing the embedding dimension based on the amount of information entropy contained within each categorical feature. Yin, et al. \cite{yin2018dimensionality} used perturbation theory for matrix factorization problems to similarly analyze the effect of the embedding dimension on the quality of classification. These methods focused primarily on the choice of $D$.

Much recent work on model compression also use compositional embeddings to reduce model complexity; see \cite{shu2017compressing,chen2018learning}. Most of these approaches require learning and storing discrete codes, similar to the idea of product quantization \cite{jegou2010product}, where each category's index is mapped to its corresponding embedding indices $i \mapsto (i_1, ..., i_m)$. In order to learn these codes during training, one is required to store them, hence requiring $O(|S| D)$ parameters with only potential to decrease $D$. Since $|S| \gg D$ in recommendation systems, these approaches unfortunately remain ineffective in our setting. 

Unlike prior approaches that focus on reducing $D$, our method seeks to directly reduce the embedding size $|S|$ using fixed codes that do not require additional storage while enforcing uniqueness of the final embedding. 

Related work by Khrulkov, et al. \cite{khrulkov2019tensorized} uses the Tensor Train decomposition to compress embeddings. Their work may be interpreted as a specific operation applied to each element of the embedding table, similar to the framework described here. Whereas their work emphasizes the application of tensor decompositions to embeddings, our work focuses on the key properties of the decomposition of the embeddings, while proposing other simpler approaches for generating compositional embeddings and discussing their tradeoffs.

\section{Conclusion}\label{sec: conclusion}

Modern recommendation systems, particularly for CTR prediction and personalization tasks, handle large amounts of categorical data by representing each category with embeddings that require multiple GBs each. We have proposed an improvement for reducing the number of embedding vectors that is easily implementable and applicable end-to-end while preserving uniqueness of the embedding representation for each category. We extensively tested multiple operations for composing embeddings from complementary partitions.

Based on our results, we suggest combining the use of thresholding with the quotient-remainder trick (and compositional embeddings) in practice. The appropriate operation depends on the network architecture; in these two cases, the element-wise multiplication operation appear to work well. This technique has been incorporated into the DLRM implementation available on Github\footnote{See \hyperlink{https://github.com/facebookresearch/dlrm}{https://github.com/facebookresearch/dlrm}}.

This work provides an improved trick for compressing embedding tables by reducing the number of embeddings in the recommendation setting, with room for design of more intricate operations. There are, however, general weaknesses of this framework; it does not take into account the frequency of categories or learn the intrinsic structure of the embeddings as in codebook learning. Although this would be ideal, we found that categorical features for CTR prediction or personalization are far less structured, with embedding sizes that often prohibit the storage of an explicit codebook during training. It remains to be seen if other compression techniques that utilize further structure within categorical data (such as codebook learning \cite{shu2017compressing,chen2018learning}) can be devised or generalized to end-to-end training and inference for CTR prediction.

\begin{acks}
We thank Tony Ginart, Jianyu Huang, Krishnakumar Nair, Jongsoo Park, Misha Smelyanskiy, and Chonglin Sun for their helpful comments. We also express our gratitude to Jorge Nocedal for his consistent support and encouragement.
\end{acks}

\bibliographystyle{ACM-Reference-Format}
\balance
\bibliography{bibl}

\appendix

\clearpage

\section{Background on Set Partitions, Equivalence Relations, and Equivalence Classes}

For completeness, we include the definitions of set partitions, equivalence relations, and equivalence classes, which we use extensively in the paper. 

\begin{definition}
Given a set $S$, a set partition $P$ is a family of sets such that: 
\begin{enumerate}
    \item $\emptyset \not\in P$
    \item $\bigcup_{A \in P} A = S$
    \item $A \cap B = \emptyset ~$ for all $A, B \in P$ where $A \neq B$.
\end{enumerate}
\end{definition}

\begin{definition}
A binary relation $\sim$ on a set $S$ is an equivalence relation if and only if for $a, b, c \in S$,
\begin{enumerate}
    \item $a \sim a$
    \item $a \sim b$ if and only if $b \sim a$
    \item If $a \sim b$ and $b \sim c$, then $a \sim c$.
\end{enumerate}
\end{definition}

\begin{definition}
Given an equivalence relation $R$ on $S$, the equivalence class of $a \in S$ is defined as 
$$[a]_R = \{b : (a, b) \in R \}.$$ 
\end{definition}

Given a set partition $P$, we can define an equivalence relation $R$ on $S$ defined as $(a, b) \in R$ if and only if $a, b \in A$ such that $A \in P$. (One can easily show that this binary relation is indeed an equivalence relation.) In words, $a$ is ``equivalent'' to $b$ if and only if $a$ and $b$ are in the same set of the partition. This equivalence relation yields a set of equivalence classes consisting of $[a]_P = [a]_R$ for $a \in S$.

As an example, consider a set of numbers $S = \{0, 1, 2, 3, 4\}$. One partition of $S$ is $P = \{ \{0, 1, 3\}, \{2, 4\} \}$. Then the equivalence classes are defined as $[0]_P = [1]_P = [3]_P = \{0, 1, 3\}$ and $[2]_P = [4]_P = \{2, 4\}$.

To see how this is relevant to compositional embeddings, consider the following partition of the set of categories $S$:
\begin{equation}
P = \left\{ \{x \in S : \varepsilon(x) \bmod m = l \} : l \in \sN \right\}
\end{equation}
where $m \in \sN$ is given. This partition induces an equivalence relation $R$ as defined above. Then each equivalence class $[x]_R$ consists of all elements whose remainder is the same, i.e. 
\begin{equation}
[x]_R = \{y : \varepsilon(y) \bmod m = \varepsilon(x) \bmod m\}.
\end{equation}
Mapping each equivalence class of this partition to a single embedding vector is hence equivalent to performing the hashing trick, as seen in Section \ref{sec: simple example}.

\section{Proof of Complementary Partition Examples}

In this section, we prove why each of the listed family of partitions are indeed complementary. Note that in order to show that these partitions are complementary, it is sufficient to show that for each pair of $x, y \in S$, there exists a partition $P$ such that $[x]_P \neq [y]_P$. 

\begin{enumerate}
    \item If $P = \{ \{x\} : x \in S \}$, then $P$ is a complementary partition. 
    \begin{proof}
    Note that since all $x \in S$ are in different sets by definition of $P$, $[x]_P \neq [y]_P$ for all $x \neq y$.
    
    \end{proof}
    \item Given $m \in \sN$, the partitions
    \begin{align*}
    P_1 & = \{ \{ x \in S : \varepsilon(x) \backslash m = l \} : l \in \mathcal{E}(\ceil{|S|/m}) \} \\
    P_2 & = \{ \{ x \in S : \varepsilon(x) \bmod m = l \} : l \in \mathcal{E}(m) \}
    \end{align*} 
    are complementary. 
    \begin{proof}
    Suppose that $x, y \in S$ such that $x \neq y$ and $[x]_{P_1} = [y]_{P_1}$. (If $[x]_{P_1} \neq [y]_{P_1}$, then we are done.) Then there exists an $l \in \sN$ such that $\varepsilon(x) \backslash m = l$ and $\varepsilon(y) \backslash m = l$. In other words, $\varepsilon(x) = ml + r_x$ and $\varepsilon(y) = ml + r_y$ where $r_x, r_y \in \{0, ..., m - 1\}$. Since $x \neq y$, we have that $\varepsilon(x) \neq \varepsilon(y)$, and hence $r_x \neq r_y$. Thus, $\varepsilon(x) \bmod m = r_x$ and $\varepsilon(y) \bmod m = r_y$, so $[x]_{P_2} \neq [y]_{P_2}$.
    
    \end{proof}
    \item Given $m_i \in \sN$ for $i = 1, ..., k$ such that $|S| \leq \prod_{i = 1}^k m_i$ , we can recursively define complementary partitions 
    \begin{align}
        P_1 & = \left\{ \{x \in S : \varepsilon(x) \bmod m_1 = l\} : l \in \mathcal{E}(m_1) \right\} \label{eq:gen div-rem base}\\ 
        P_j & = \left\{ \{ x \in S : \varepsilon(x) \backslash M_j \bmod m_j = l \} : l \in \mathcal{E}(m_j) \right\} \label{eq:gen div-rem recurse}
    \end{align}
    where $M_j = \prod_{i = 1}^{j - 1} m_i$ for $j = 2, ..., k$. Then $P_1, P_2, ..., P_k$ are complementary.
    \begin{proof}
    We can show this by induction. The base case is trivial. 
    
    Suppose that the statement holds for $k$, that is if $|S| \leq \prod_{i = 1}^k m_i$ for $m_i \in \sN$ for $i = 1, ..., k$ and $P_1, P_2, ..., P_k$ are defined by \eqref{eq:gen div-rem base} and \eqref{eq:gen div-rem recurse}, then $P_1, P_2, ..., P_k$ are complementary. We want to show that the statement holds for $k + 1$.
    
    Consider a factorization with $k + 1$ elements, that is, if $|S| \leq \prod_{i = 1}^{k + 1} m_i$ for $m_i \in \sN$ for $i = 1, ..., k + 1$. Let $P_j$ for $j = 1, ..., k$ be defined as in \eqref{eq:gen div-rem base} and \eqref{eq:gen div-rem recurse} with
    \begin{align}
        S_1^l & = \{x \in S : \varepsilon(x) \bmod m_1 = l \} \\
        S_j^l & = \{x \in S : \varepsilon(x) \backslash M_j  \bmod m_j = l \}
    \end{align}
    for all $l = 0, ..., m_j - 1$ and $j = 1, ..., k$.
    
    Let $y, z \in S$ with $y \neq z$. Since $y \neq z$, $\varepsilon(y) \neq \varepsilon(z)$. We want to show that $[y]_{P_i} \neq [z]_{P_i}$ for some $i$. We have two cases:
    
    \begin{enumerate}
        \item If $[y]_{P_{k + 1}} \neq [z]_{P_{k + 1}}$, then we are done.
        \item Suppose $[y]_{P_{k + 1}} = [z]_{P_{k + 1}}$. Then
        \begin{equation*}
            \varepsilon(x) \backslash M_j \bmod m_j = \varepsilon(y) \backslash M_j \bmod m_j = \tilde{l}
        \end{equation*}
        for some $\tilde{l} \in \mathcal{E}(m_{k + 1})$. Consider the subset
        \begin{equation*}
            \tilde{S} = \left\{ x \in S : \varepsilon(x) \backslash M_{k + 1} \bmod m_{k + 1} = \tilde{l} \right\}.
        \end{equation*}
        Note that $|\tilde{S}| \leq \prod_{i = 1}^k m_i$. We will define an enumeration over $\tilde{S}$ by
        \begin{equation*}
            \tilde{\varepsilon}(x) = \varepsilon(x) - \tilde{l} \prod_{i = 1}^k m_i.
        \end{equation*}
        Note that this is an enumeration since if $x \in \tilde{S}$, then
        \begin{equation*}
            \varepsilon(x) = \tilde{l} \prod_{i = 1}^k m_i + c
        \end{equation*} 
        for $c \in \{0, ..., \prod_{i = 1}^k m_i - 1\}$, so $\tilde{\varepsilon}(x) = c \in \{0, ..., \prod_{i = 1}^k m_i - 1\}$. This function is clearly a bijection on $\tilde{S}$ since $\varepsilon$ is a bijection. Using this new enumeration, we can define the sets 
        \begin{align*}
            \tilde{S}_1^l & = \{ x \in S : \varepsilon(x) \bmod m_1 = l \} \\
            \tilde{S}_j^l & = \{ x \in S : \varepsilon(x) \backslash M_j \bmod m_j = l \}
        \end{align*}
        where $M_j = \prod_{i = 1}^{j - 1} m_i$ for $l \in \mathcal{E}(m_j)$ and $j = 2, ..., k$, and their corresponding partitions
        \begin{align*}
            \tilde{P}_1 & = \{ \tilde{S}_1^l :  l \in \mathcal{E}(m_1) \} \\
            \tilde{P}_j & = \{ \tilde{S}_j^l : l \in \mathcal{E}(m_j) \}.
        \end{align*}
        Since $|\tilde{S}| \leq \prod_{i = 1}^k m_i$, by the inductive hypothesis, we have that this set of partitions are complementary. Thus, there exists an $i$ such that $[y]_{\tilde{P}_i} \neq [z]_{\tilde{P}_i}$.

        In order to show that this implies that $[y]_{P_i} \neq [z]_{P_i}$, one must show that $\tilde{S}_j^l \subseteq S_j^l$ for all $l = 0, ..., m_j - 1$ and $j = 1, ..., k$.
        
        To see this, since $\varepsilon(x) = \tilde{l} \prod_{i = 1}^k m_i + \tilde{\varepsilon}(x)$, by modular arithmetic we have
        \begin{equation*}
            \varepsilon(x) \bmod m_1 = \tilde{\varepsilon}(x) \bmod m_1
        \end{equation*}
        and
        \begin{align*}
            & \varepsilon(x) \backslash M_j \bmod m_j \\  
            & ~~~ = \varepsilon(x) \backslash \left( \prod_{i = 1}^{j - 1} m_i \right) \bmod m_j \\ 
            & ~~~ = \left( \tilde{l} \prod_{i = 1}^k m_i + \tilde{\varepsilon}(x) \right) \backslash \left( \prod_{i = 1}^{j - 1} m_i \right) \bmod m_j \\
            & ~~~ = \tilde{\varepsilon}(x) \backslash \left( \prod_{i = 1}^{j - 1} m_i \right) \bmod m_j
        \end{align*}
        for $j = 2, ..., k$. Thus, $\tilde{S}_j^l \subseteq S_j^l$ for all $l = 0, ..., m_j - 1$ and $j = 1, ..., k$ and we are done. 
        
    \end{enumerate}
    
    \end{proof}
    \item Consider a pairwise coprime factorization greater than or equal to $|S|$, that is, $|S| \leq \prod_{i = 1}^k m_i$ for $m_i \in \sN$ for all $i = 1, ..., k$ and $\gcd(m_i, m_j) = 1$ for all $i \neq j$. Then we can define the partitions
    \begin{equation*}
        P_j = \{ \{ x \in S : \varepsilon(x) \bmod m_j = l \} : l \in \mathcal{E}(m_j) \}
    \end{equation*}
    for $j = 1, ..., k$. Then $P_1, P_2, ..., P_k$ are complementary. 
    \begin{proof}
    Let $M = \prod_{i = 1}^k m_i$. Since $m_i$ for $i = 1, ..., k$ are pairwise coprime and $|S| \leq M$, by the Chinese Remainder Theorem, there exists a bijection $f : \sZ_{M} \rightarrow \sZ_{m_1} \times ... \times \sZ_{m_k}$ defined as $f(i) = (i \bmod m_1, ..., i \bmod m_k)$. Let $x, y \in S$ such that $x \neq y$. Then $f(\varepsilon(x)) \neq f(\varepsilon(y))$, and so there must exist an index $i$ such that $\varepsilon(x) \bmod m_i \neq \varepsilon(y) \bmod m_i$, as desired. Hence $[x]_{P_i} \neq [y]_{P_i}$.
    
    \end{proof}
\end{enumerate}

\section{Proof of Theorem \ref{thm: uniqueness}}

\setcounter{theorem}{0}
\begin{theorem}
Assume that vectors in each embedding table $\mW_j = \left[ \vw_1^j, ..., \vw_{|P_j|}^j \right]^T$ are distinct, that is $\vw_i^j \neq \vw_{\hat{i}}^j$ for $i \neq \hat{i}$ for all $j = 1, ..., k$. If the concatenation operation is used, then the compositional embedding of any category is unique, i.e. if $x, y \in S$ and $x \neq y$, then $\vx_{\text{emb}} \neq \vy_{\text{emb}}$.
\end{theorem}

\begin{proof}
Suppose that $x, y \in S$ and $x \neq y$. Let $P_1, P_2, ..., P_k$ be complementary partitions. Define $\mW_1 \in \R^{|P_1| \times D_1}, \mW_2 \in \R^{|P_2| \times D_2}, ..., \mW_k \in \R^{|P_k| \times D_k}$ to be their respective embedding tables. Since the concatenation operation is used, denote their corresponding final embeddings as 
\begin{align*}
    \vx_{\text{emb}} &= [\vx_1^T, \vx_2^T, ..., \vx_k^T]^T \\
    \vy_{\text{emb}} &= [\vy_1^T, \vy_2^T, ..., \vy_k^T]^T
\end{align*}
respectively, where $\vx_1, \vy_1 \in \R^{D_1}, \vx_2, \vy_2 \in \R^{D_2}, ..., \vx_k, \vy_k \in \R^{D_k}$ are embedding vectors from each corresponding partition's embedding table.  

Since $P_1, P_2, ..., P_k$ are complementary and $x \neq y$, $[x]_{P_j} \neq [y]_{P_j}$ for some $j$. Thus, since the embedding vectors in each embedding table is distinct, $\vx_j = (\mW_j)^T \ve_{p_j(x)} \neq (\mW_j)^T \ve_{p_j(y)} = \vy_j$. Hence, $\vx_{\text{emb}} \neq \vy_{\text{emb}}$, as desired.

\end{proof}

\end{document}